\def\year{2021}\relax
\documentclass[letterpaper]{article} 
\usepackage{aaai21}  
\usepackage{times}  
\usepackage{helvet} 
\usepackage{courier}  
\usepackage[hyphens]{url}  
\usepackage{graphicx} 
\usepackage{natbib}  
\usepackage{caption} 
\usepackage{cite}
\usepackage{amsmath,amssymb,amsfonts,amsthm}
\usepackage{algorithm,algorithmicx, algpseudocode}
\usepackage{graphicx}
\usepackage{textcomp}
\usepackage{xcolor}
\usepackage{soul}
\usepackage{multirow}
\usepackage[flushleft]{threeparttable}

\nocopyright
\title{HydaLearn: Highly Dynamic Task Weighting for Multi-task Learning with Auxiliary Tasks}
\author{
    Sam Verboven,\textsuperscript{\rm 1}
    Muhammad Hafeez Chaudhary, \textsuperscript{\rm 2}
    Jeroen Berrevoets \textsuperscript{\rm 1}
    Wouter Verbeke \textsuperscript{\rm 3}\\
}
\affiliations{
    \textsuperscript{\rm 1} Vrije Universiteit Brussel, Belgium\\
    \textsuperscript{\rm 2} Royal Military Academy, Belgium \\
    \textsuperscript{\rm 3} Katholieke Universiteit Leuven, Belgium \\
    \{sam.verboven, jeroen.berrevoets\}@vub.be, mh.chaudhary@rma.ac.be, wouter.verbeke@kuleuven.be
}

\begin{document}

\newcommand{\hafeez}[1]{\textbf{\color{brown}[HC - #1]}}
\newcommand{\jeroen}[1]{\textbf{\color{magenta}[JB - #1]}}
\newcommand{\remove}[1]{\textbf{\color{red}\st{#1}}}
\newcommand{\sam}[1]{\textbf{\color{blue}[SV - #1]}}
\newcommand{\wouter}[1]{\textbf{\color{green}[SV - #1]}}

\newcommand{\mc}{\mathcal}

\newtheorem{theorem}{Theorem}




\maketitle

\begin{abstract}
Multi-task learning (MTL) can improve performance on a task by sharing representations with one or more related auxiliary-tasks. Usually, MTL-networks are trained on a composite loss function formed by a constant weighted combination of the separate task losses. In practice, constant loss weights lead to poor results for two reasons: (i) the relevance of the auxiliary tasks can gradually drift throughout the learning process; (ii) for mini-batch based optimisation, the optimal task weights vary significantly from one update to the next depending on mini-batch sample composition. We introduce HydaLearn, an intelligent weighting algorithm that connects main-task gain to the individual task gradients, in order to inform dynamic loss weighting at the mini-batch level, addressing i and ii. Using HydaLearn, we report performance increases on synthetic data, as well as on two supervised learning domains.
\end{abstract}

\section{Introduction.}
Through joint training of shared representations with one or more auxiliary tasks, Multi-task learning \cite{caruana1997multitask} can increase performance of neural networks on a task of interest - i.e. the main task.

\textbf{The Problem Setting.} \quad How much a given auxiliary task should influence the training process at each step is an open research question. The gain of learning from auxiliary tasks depends on their back propagated gradients that contribute in learning the main task. While training on a composite multi-task loss has the potential to provide a richer, better regularized training signal, MTL networks are not straightforward to train. Auxiliary tasks in fact do not always contribute to better predictions for the main task. In particular, when the signal is not sufficiently relevant, performance on the main task can deteriorate. Moreover, the usefulness of the auxiliary task gradients is subject to changes over the course of training process \cite{du2018adapting}.  Automated adaptive MTL is as such a high-impact problem, for which a solution would mitigate a significant part of the difficulty of training powerful MTL models. Consequently, this could lead to increased adoption of such models among practitioners. More wide-spread diffusion could occur specifically outside of high-dimensional, high-compute supervised learning problems.

\textbf{Slow Weight Adjustment is Not Enough.} \quad 
Contemporary research on adaptive task weighting either does not allow explicit prioritization of main task performance \cite{sener2018multi, chen2017gradnorm, ruder2019latent}, or requires some implicit assumptions which are frequently violated. For instance: (i) the main task gradient direction is consistently desirable \cite{du2018adapting} or; (ii) task usefulness only changes gradually over the learning process \cite{lin2019adaptive}. However, the composition of the mini-batch ultimately determines the training signal, and is expected to be highly variable. This variability is desirable for stochastic optimization of non-convex loss surfaces \cite{hochreiter1997flat}, but causes the optimal loss weighting in multi-task learning to drastically change from mini-batch to mini-batch. Accordingly, the usefulness of a single back propagated task gradient is not only dependent on: (i) the parametrization of the model; but, we argue, (ii) the sample composition of the mini-batch.
Contrary to the slower changing parametrization of the model, mini-batch statistics can differ significantly from one update to the next. Hence, slow weight adjustment over the course of training does not suffice. 

\begin{figure}[t] 
	\begin{center}
		\includegraphics[width=\columnwidth]{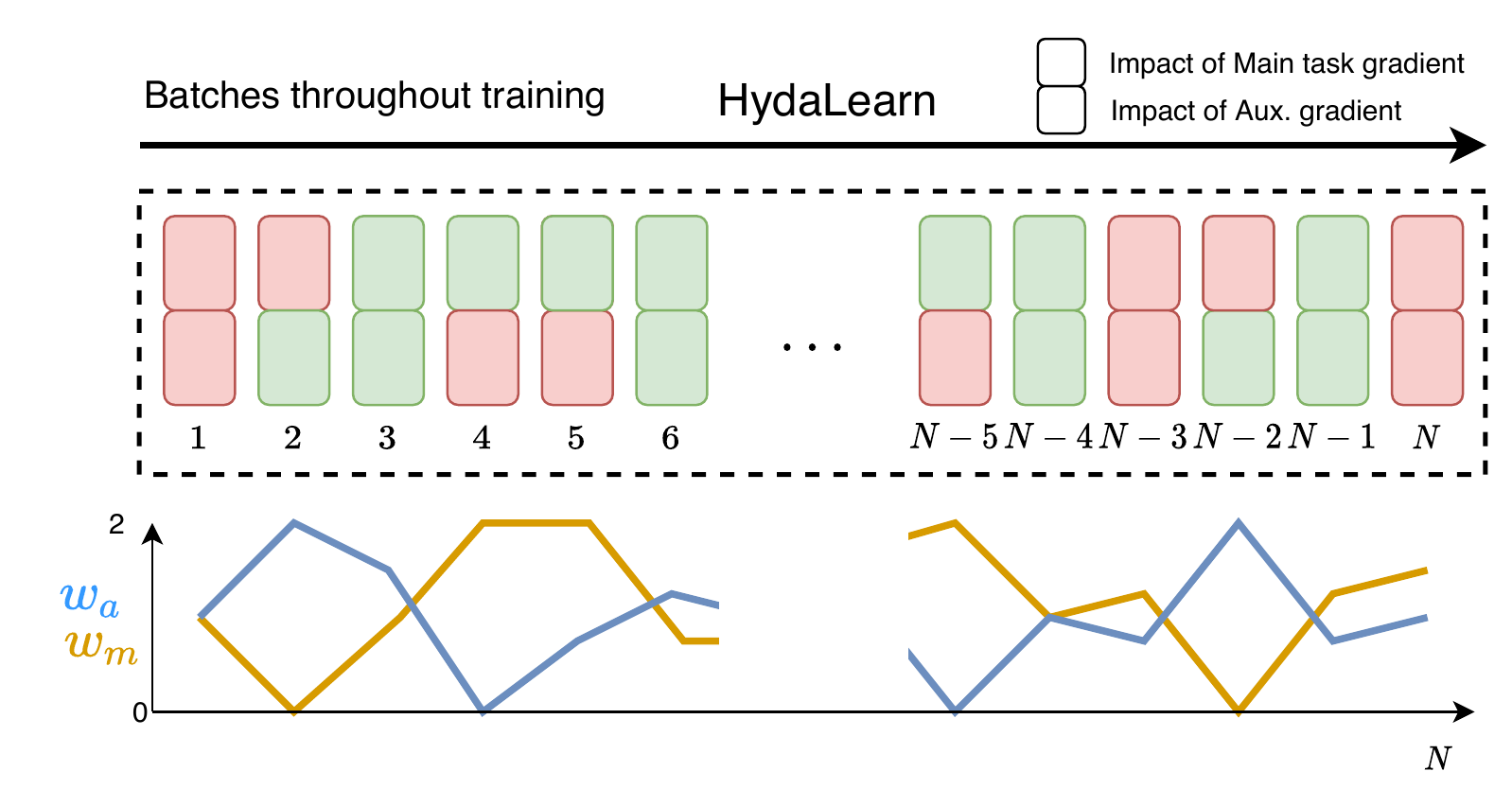}
		\caption{A green box indicates \textbf{main task} performance would improve through updating the respective loss for the mini-batch under review. Conversely, red indicates main task performance degradation}
		\label{fig_mtl_basic}
	\end{center}
\end{figure}

\textbf{Our Contribution.} \quad
We introduce and evaluate HydaLearn; an adaptive task weighting algorithm designed to handle high variance estimates of the exact gradients. On the level of the individual mini-batch, we estimate the usefulness of individual task gradients based on a metric of choice. Subsequently, these usefulness measures are used to compute task weights which are in turn used to construct a composite gradient. This composite gradient then updates the model parameters. 

We extensively evaluate HydaLearn on a synthetic toy example as well as two real-world datasets, and demonstrate the effectiveness of the algorithm compared to a both logical and state-of-the-art baselines. The latter datasets comprise two separate supervised learning tasks. First, we perform in-hospital mortality prediction, aided by length-of-stay prediction using the MIMIC-III dataset \cite{johnson2016mimic}. Second, we cast default prediction as MTL through inclusion of prepayment prediction as an auxiliary task, using the Fannie Mae Single Family Loan Performance dataset.

\section{Related Literature}

\textbf{Multi-task Learning.}
The benefit from MTL is dependent on sample size, number of tasks, and the intrinsic data dimensionality \cite{maurer2016benefit}. The dominant paradigm to reap this benefit is hard parameter sharing \cite{Caruana1993MultitaskLA}. In hard parameter sharing \cite{Caruana1993MultitaskLA, ruder2016overview}, all parameters in a set of hidden layers are fully shared between the tasks. This fully shared hidden layer stack (or encoder) learns a shared representation, and is generally followed by one or more layers which do not share parameters, i.e., task-specific layers. These task-specific \emph{heads} of the network are often called decoders and learn a the task-specific representation. 

In MTL, a multitask loss is usually used, composed of a linearly weighted combination of individual task losses. These weights are constants and generally chosen through expensive search procedures \cite{harutyunyan2019multitask} or Bayesian optimization \cite{guo2019autosem}. 

The objective in MTL can be either to simultaneously maximise performance of all tasks \cite{sener2018multi, chen2017gradnorm}, or leverage auxiliary tasks insofar they improve the main task \cite{lin2019adaptive, du2018adapting, zhang2014facial, ruder2019latent}, i.e., not care about the performance on the auxiliary tasks. The latter objective is the focus of this work.

\textbf{Auxiliary Task Learning.} \quad
It is generally assumed that related tasks will improve and unrelated tasks can hamper performance \cite{bingel2017identifying, rai2010infinite}, although even unrelated tasks can be exploited \cite{romera2012exploiting}. However, there is still no theoretically grounded definitive definition of task relatedness. Empirically, adding auxiliary tasks has shown to improve performance for \cite{guo2019autosem}. The concept of adding auxiliary tasks is most common in complex, high dimensional domains such as autonomous vehicle control \cite{yang2018end}, reinforcement learning \cite{du2018adapting, lin2019adaptive},  and natural language processing \cite{collobert2008unified}. 

However, MTL can also be a valid approach on common low-compute problems.
For example, the implicit data augmentation inherent to MTL can help to learn from imbalanced data \cite{caruana1996using}, which is commonly encountered in fields such as fraud detection \cite{baesens2015fraud} and default prediction \cite{baesens2005neural}. Different strands of work have proposed strategies for auxiliary task selection or training. Our work contributes to the latter, more specifically by dynamically learning the task loss weights.

\textbf{Adaptive Task Weighting in Multi-task Learning}. \quad
The optimal task loss distribution changes as parametrization of the model gradually changes. To handle this, task weighting should be dynamically adapted throughout the training cycle. 
Most of the existing work on adaptive task learning considers the setting where one tries to optimize all the included tasks together \cite{chen2017gradnorm, kendall2018multi, jean2019adaptive, sener2018multi}. Different approaches have been proposed including
explicit prioritization of difficult tasks \cite{guo2018dynamic} , minimization of negative conflict between gradients \cite{yu2020gradient, du2018adapting}, and balancing of the task gradient norms \cite{chen2017gradnorm}. Furthermore, homeostatic uncertainty can be used to inform task weighting \cite{kendall2018multi}. Finally, \cite{jean2019adaptive} focus on adapting the learning rates of the tasks, which under vanilla stochastic gradient descent (SGD) is equivalent to scaling the gradients.

\textbf{Adaptive Task Weighting for Auxiliary Tasks}. \quad
The algorithms most relevant to ours are proposed in \cite{lin2019adaptive} and \cite{du2018adapting} and aim to explicitly maximize performance on the main task \cite{lin2019adaptive, du2018adapting}. \citet{du2018adapting} only train on the auxiliary task if its gradient aligns with the main task gradient, as defined by the so-called \emph{cosine similarity} measure. In \citet{lin2019adaptive}, the relative task weights of auxiliary tasks are updated every $N$ steps using a variant of online cross-validation \cite{sutton1992adapting}, where $N$ is a hyper-parameter. 

\textbf{Small Mini-Batch Learning}. \quad
Neural networks are usually trained with the mini-batch SGD method \cite{bottou2010large} or one of its variants \cite{kingma2014adam, ruder2016overview}. A common belief among researchers dictates that smaller batch sizes yield better out-of-sample generalization \cite{smith2017bayesian}. This belief has seen theoretical and experimental validation \cite{mccandlish2018empirical,keskar2016large, zhang2018theory, golmant2018computational}, although it has also been challenged \cite{goyal2017accurate, hoffer2017train}. Nonetheless, standard practice in most domains is still to train with small mini-batches. 

\textbf{Variable Usefulness of Mini-Batches}. \quad 
When training a neural network, some data points are more valuable than others \cite{pang2019libra, li2019gradient}. By extension, through the stochastic nature of mini-batch SGD, the number of valuable or harmful examples varies over mini-batches, i.e., some mini-batches are more valuable than others. Harder examples tend to give larger gradients, and are in some cases more useful \cite{lin2017focal, oksuz2020imbalance}. Nonetheless, it has been shown that harder does not always imply more valuable \cite{li2019gradient}. Furthermore, the usefulness of individual examples can differ between tasks. Ultimately, the goal is to optimize performance on the main task. This implies that the gradient norm distribution of the auxiliary task is not informative for the usefulness of individual examples, and by extension batches. Accordingly, strategies which try to exploit such information \cite{li2019gradient} are not applicable in this context.

\section{HydaLearn}\label{sec_alg_prop}

\begin{figure*}[t] 
	\begin{center}
		\includegraphics[scale=0.85]{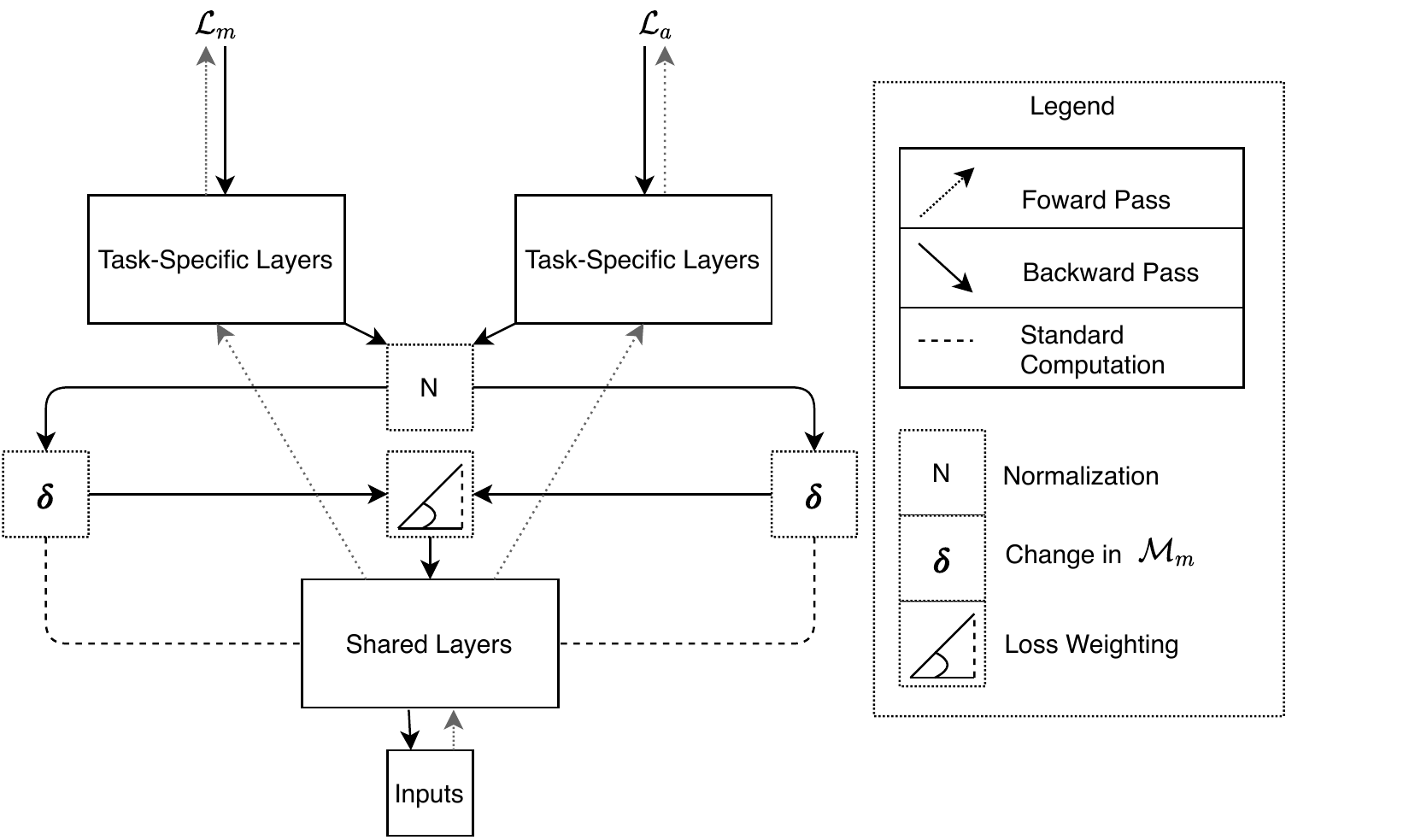}
		\caption{Diagram displaying the forward and backward pass for HydaLearn}
		\label{fig_hydralearn_forwback}
	\end{center}
\end{figure*}

\textbf{Problem Formulation.}
We focus on a problem comprising a main task $\mc T_m$ and an auxiliary task $\mc T_a$. As shown in Fig. \ref{fig_hydralearn_forwback}, we have a number of shared layers between the task and then we have task-specific layers.   Let $\mc L_m$ and $\mc L_a$ be the loss function associated with the main and auxiliary task respectively. The two losses are combined to form a total loss as follows: 
 \begin{align}
 \mc L(\theta_{s,t}, \theta_{m,t}, \theta_{a,t}) =& w_{m, t} \mc L_m(\theta_{s,t}, \theta_{m,t}) 
\nonumber\\ 
 &+ w_{a, t}\mc L_a(\theta_{s,t}, \theta_{a,t}),
 \label{eq_totloss}
 \end{align}
In \eqref{eq_totloss}, $t$ is the training step; $w_{m, t}, w_{a, t} \in \mathbb{R}_+$ denote the task weights, where $\mathbb{R}_+$ is a set of all positive real numbers; and $\theta_{s,t}, \theta_{m,t}$, $\theta_{a,t}$ denote model parameters of the shared and task-specific layers. We aim to find a solution for training the model parameters and develop a method to update task weights intelligently. 

The model parameters are updated using mini-batch SGD with the objective of minimizing total loss, for given task weights. Let $\theta_t \triangleq [\theta_{s,t}, \theta_{m,t}$, $\theta_{a,t}]$. The gradient descent update at training step $t$ can be written as: 
\begin{align}
 \theta_{t+1} = \theta_t  - \alpha \nabla_{\theta_t} \mc L(\theta_t), 
 \label{eq_gradstep}
\end{align}
where $\alpha$ denotes the gradient descent step-size, also called the learning rate. 

\textbf{Proposed Solution.}
The key innovation of HydaLearn lies in the way we determine the optimal loss weights for each mini-batch separately. For optimizing the task weights $w_{m, t}$ and $w_{a, t}$, we focus on maximizing gain on a given metric for the main task. Let  $\mc M_m$ denote the metric function for the main task. 
This metric is calculated over a much larger set of examples than the batch size, and is thus a more stable guide of progression. At the same time we do not fully block 'bad' gradient steps, ensuring sufficient stochasticity in the learning path to generalize well.
Furthermore, the choice of metric function is not constrained to any particular function type. This implies that non-differentiable metrics can be chosen, but also the corresponding task loss itself. Recent dynamic weight adaption methods (e.g., \cite{du2018adapting, li2019gradient}) require the underlying function to be strictly differentiable. Finding an approximation that is representative of the dynamics of the underlying function over the support region of interest is not straightforward. 

In the following theorem we first establish the relationship between the task weights and the gain on the metric $\mc M_m$. Next, in the ensuing paragraphs, we use this expression to develop the HydaLearn algorithm.

\begin{theorem}
Let $\delta_{m, m, t}$ and $\delta_{m, a, t}$ denote the gain computed on the main task metric at training step $t$, with gradient-descent steps executed separately on the main and auxiliary task-loss objectives. More concrete mathematical definition of the $\delta$'s is given in Appendix \ref{Ch6:Appendices:TBD}.  The following relationship holds between the gain values and the task weights:
\begin{align}
 \dfrac{w_{m, t}}{w_{a, t}} \approx \dfrac{\delta_{m, m, t}}{\delta_{m, a, t}}.
 \label{eq_delta_frac}
\end{align}
\end{theorem}

\begin{proof}
See Appendix \ref{Ch6:Appendices:TBD}. 
\end{proof}
 
  With known values of $\delta$'s, and constraining the sum of the task weights, we can compute the weights from \eqref{eq_delta_frac}. We impose a constraint on the combined weights, given as $w_m + w_a = W$. This constraint serves two purposes. First, it makes the solution to the weight optimization problem easy to find\textemdash{two linear equations to solve for two unknowns}. Second, it helps to restrict the total learning rate to $W$ in the gradient descent algorithm in \eqref{eq_gradstep}. The latter constraint is also used in other work, e.g., \cite{chen2017gradnorm}. It is interesting to note that the result in \eqref{eq_delta_frac} does not depend on any derivative of the loss and the metric functions and the computations of $\delta$'s in our proposed algorithm are based on metric function $\mc M_m$ values from two updates, further explained in the ensuing sections.
 
 The overall procedure to train the model parameters and the task weights is described in Algorithm \ref{Alg_HydaLearn}. We call this algorithm HydaLearn, or \emph{Highly Dynamic Learning}. The weight updating procedure in the algorithm aims to optimize the metric function $\mc M_m$ at each training step. The weights can adapt to the signal coming from the individual task losses, to realize maximal gain on the metric. Extension to cases where the metric function needs to be minimized is straightforward. Furthermore, one can choose to optimize the metric on (a subset of) either the training or validation set. 

\textbf{Analysis.} 
The current HydaLearn solution is applicable to two-task settings\textemdash{a main and an auxiliary task}. Extension to cases that entail more than one auxiliary task is planned in upcoming work.

\begin{algorithm}[t]
\caption{HydaLearn algorithm}\label{Alg_HydaLearn}
\begin{algorithmic}[1]
    \State \textbf{Input}:
    \State \quad Main task loss $\mc L_m$ and auxiliary task loss $\mc L_a$
    \State \quad Main task metric $\mc M_m$
    \State \quad Learning rate $\alpha$
    \State \quad Total weight $W$
    \State \quad Select dataset on which to compute metric $\mc M_m$
    \State 
    \State \textbf{Initialize}:
    \State \quad $\theta_0$, $\mu_0$, $w_m=w_a=W/2$ and $N$ total training steps
    \State
    \Function{ComputeTaskWeight}{} \Comment{Function for computing task weights}
        \State \textbf{Compute} $\delta_{m, m}$:
        \State \quad $\theta'_{s, t+1} \gets \theta_{s, t}  - \alpha \nabla_{\theta_{s, t}} \mc L_m(\theta_{s,t}, \theta_{m,t})$ 
        \State \quad $\mu_{m, m, t+1} \gets$ Compute $M_m(\theta'_{s, t+1}, \theta_{m, t+1})$
        \State  \quad $\delta_{m, m, t+1}\gets \mu_{m, m, t+1} - \mu_t$ 
        \State \textbf{Compute} $\delta_{m, a}$:
        \State \quad $\theta'_{s, t+1} \gets \theta_{s, t}  - \alpha \nabla_{\theta_{s, t}} \mc L_a(\theta_{s,t}, \theta_{a,t})$
        \State \quad $\mu_{m, a, t+1} \gets$ Compute $M_m(\theta'_{s, t+1}, \theta'_{m, t+1})$ 
        \State  \quad $\delta_{m, a, t+1}\gets  \mu_{m, a, t+1} - \mu_t$ 
        \State \textbf{Update} $w_m$ and $w_a$ such that
        $\tfrac{w_{m, t+1}}{w_{a, t+1}} = \tfrac{\delta_{m, m, t+1}}{\delta_{m, a, t+1}}$ and $w_{m, t+1} + w_{a, t+1} = W$ hold.  
        \If{$\delta_{m, m, t+1} \geq \delta_{m, a, t+1}$}
            \State $\mu_{t+1} \gets \mu_{m, m, t+1}$
        \Else 
            \State $\mu_{t+1} \gets \mu_{m, a, t+1}$
        \EndIf
    \EndFunction
    \State 
    \For {$t=1$ to $N$}  \Comment{Main training loop}
        \State Sample a mini-batch from the training data-set
        \State \textbf{Update model parameters of task specifics layers}:
        \State \quad $\theta_{m, t+1} \gets \theta_{m, t}  - \alpha \nabla_{\theta_{m, t}} \mc L_m(\theta_t)$ 
        \State \quad $\theta_{a, t+1} \gets \theta_{m, t}  - \alpha \nabla_{\theta_{a, t}} \mc L_a(\theta_t)$ 
        \State \textbf{Update task weights}: ComputeTaskWeight() function
        \State \textbf{Update model parameters of shared layers}:
        \State \quad $\mc L(\theta_t) \gets w_{m, t+1} \mc L_m(\theta_t) + w_{a, t+1}\mc L_a(\theta_t)$
        \State \quad $\theta_{t+1} \gets \theta_t  - \alpha \nabla_{\theta_t} \mc L(\theta_t)$
    \EndFor
\end{algorithmic}
\end{algorithm}

The inner working of the HydaLearn algorithm can be better understood with the help of Fig. \ref{fig_hydralearn_forwback}. The forward pass is a standard step, whereas backward pass is taken in three steps, in the following order: 
\begin{enumerate}
    \item Task specific model parameters are updated.
    \item Then for the computation of $\delta$'s we perform two, what we call, fake-updates for the shared layer model parameters\textemdash{one based on main-task loss and the second based on the auxiliary-task loss}. The update is called fake because the resulting model parameters are only used in the computation of $\delta$'s. After each fake update, the corresponding $\delta$ is computed. Based on the resulting values of $\delta$'s, task weights are computed. 
    \item Finally, the new task weights are then used for the actual update of the shared layers model parameters. 
\end{enumerate}

\textbf{Comparison with Related Approaches.} \quad
Existing methods are not geared towards handling the high-variance estimators of the gradient, i.e., mini-batches.
While the gradient cosine similarity (Gcosim) based method proposed in \cite{du2018adapting} uses information about the gradient of the main task to ground the contribution of the auxiliary task, \cite{lin2019adaptive} uses feedback from past batches to take gradient-descent steps on the task weights. When the main task gradient is noisy, Gcosim can exacerbate bad updates, or block helpful auxiliary task gradients, cf., the bottom left panel in Fig. \ref{fig_gcosim_hydalearn}. When the gradient cosine similarity is positive, there is no normalization of the gradient size, and one loss may dominate the other, cf., top panels in Fig. \ref{fig_gcosim_hydalearn}. Our algorithm handles such cases differently, as illustrated in the left panels. Even though $\mc T_a$ gradients are dominant, more weight is given to $\mc T_m$ for this particular update, due to the superior direction of its gradients. Olaux \cite{li2019gradient}, on the other hand, is not geared towards highly varying weights since: (i) it only features a single gradient step on the task weights at a time and; (ii) this gradient step towards the task weights occurs after network parameters are updated for the current batch, i.e., the weighting is not related to the contribution of the current batch.

\begin{figure}[t] 
	\begin{center}
		\includegraphics[scale=0.37]{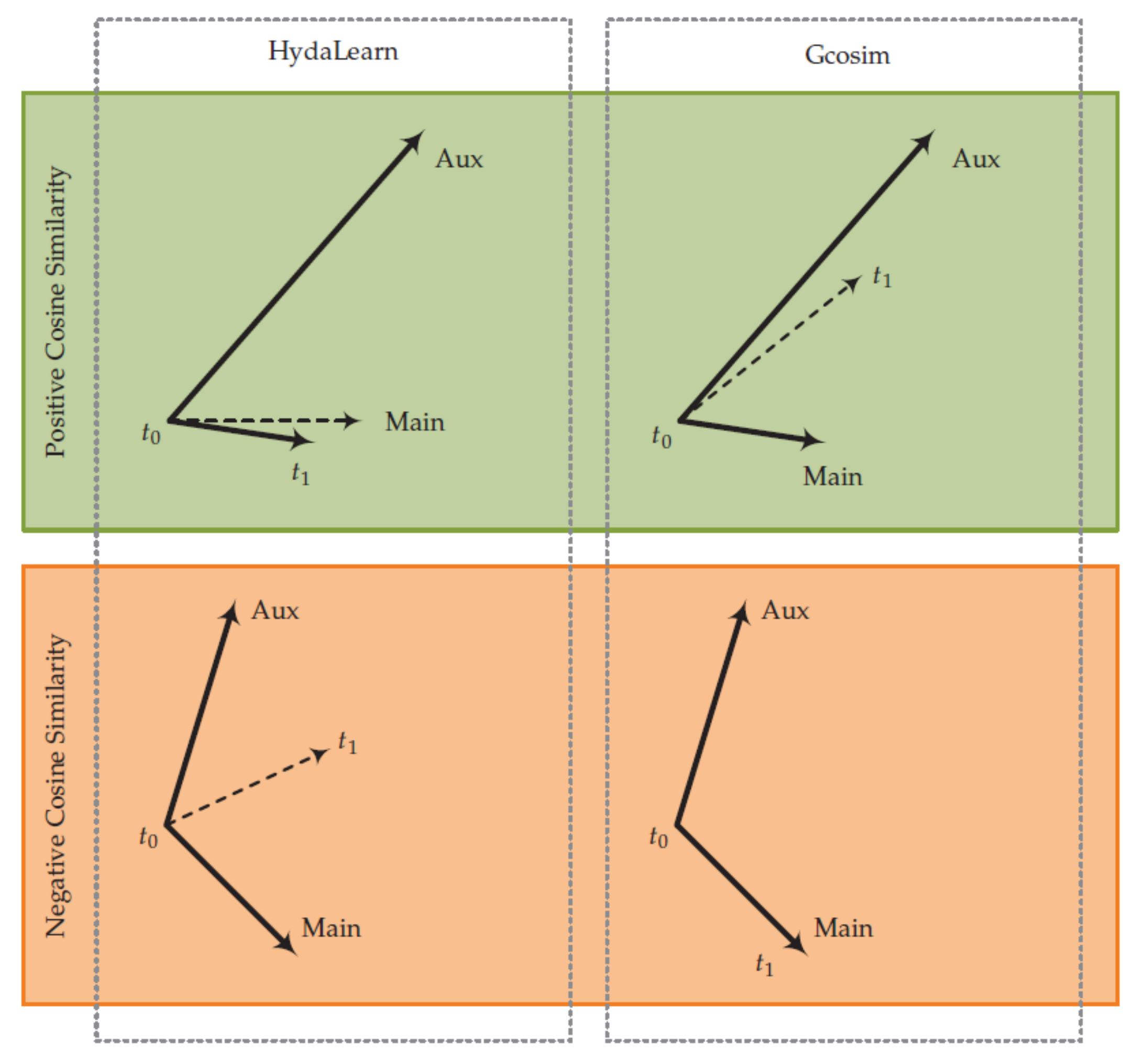}
		\caption{Conceptual diagram illustrating the difference in behaviour between HydaLearn and Gcosim.}
		\label{fig_gcosim_hydalearn}
	\end{center}
\end{figure}

In the ensuing sections we evaluate our algorithm on synthetic toy examples as well as on real world datasets. We compare results with a list of prominent algorithms, from the current state of art, that tackle similar MTL problem. There we will show that our algorithm enables redistribution of weight towards the task for which a particular mini-batch yields the most valuable signal, i.e., causes the most performance gain on the main task.

\section{Experiments}

\subsection{Baselines}
For performance comparison we have selected following algorithms as a baseline, where the first two can be viewed as standard baselines used in such comparative studies whereas the latter three are current state of the art. 
\begin{enumerate}
    \item \textbf{Single Task Model} (STL):  A model trained only on the main task
    \item \textbf{Static Loss weights} (Static): A baseline with static weights throughout the full training process
    \item \textbf{GradNorm}: GradNorm \cite{chen2017gradnorm} balances the training rates of all tasks by normalizing the gradient magnitudes through tuning of the multitask loss 
    \item \textbf{Gcosim}: This algorithm only uses the auxiliary task gradient when it aligns with the main task gradient as defined by cosine similarity \cite{du2018adapting}
    \item \textbf{OLaux}: Online cross-validation method for adaptive task weighting. The weights of the auxiliary task are updated based on past batches with gradient descent, every few steps \cite{li2019gradient}
\end{enumerate}
For each of these baselines, we use the same encoder-decoder backbone architecture. We assign an equal budget for optimization,  the learning rate and algorithm-specific parameters to ensure a fair comparison.

\subsection{Simulation Study}
Using a toy example, we first show that our model is able to efficiently learn when the usefulness of both task gradients is varying. Second, we demonstrate that our model can ignore harmful auxiliary tasks.

\textbf{Setup}.
We generate two toy examples:
\begin{description}
    \item[\textbf{Exp1}--] \textbf{Related Auxiliary Task:} The main and auxiliary tasks are sampled from the $\tanh$ function class according to \eqref{eq_toyreg}. As such, the tasks are related through their common basis $\mathbf{B}$, and the auxiliary task can help learning the main task. Furthermore, some Gaussian noise is added to both tasks. 
    
    \item[\textbf{Exp2}--] \textbf{Unrelated Auxiliary Task:} To ensure the unrelatedness of the auxiliary task, the output values for $\mc T_a$ are uniformly sampled, I.I.D., from the output range of $\mc T_m$ over the dataset, i.e., there is no systematic relatedness between the tasks. 
\end{description}
All models are trained using the same backbone. For Exp1, a 4-layer encoder, and two 1-layer task-specific decoders with 64 and 32 neurons, respectively. For Exp2, a 2-layer encoder and two 1-layer task-specific decoders with 40 and 20 neurons were employed. A small grid-search is used to optimize algorithm-specific hyper-parameters. All other hyper-parameters are kept the same for every method.
The exact implementation details for the toy examples are the corresponding models can be found in Appendix \ref{annex_toyreg}. The models are both trained and evaluated with standard mean squared error (MSE) as loss function $\mc L_m$ and $\mc L_a$,  and mean absolute error (MAE) as metric function $\mc M_m$. 

\begin{figure*}[t] 
	\begin{center}
		\includegraphics[scale=0.68]{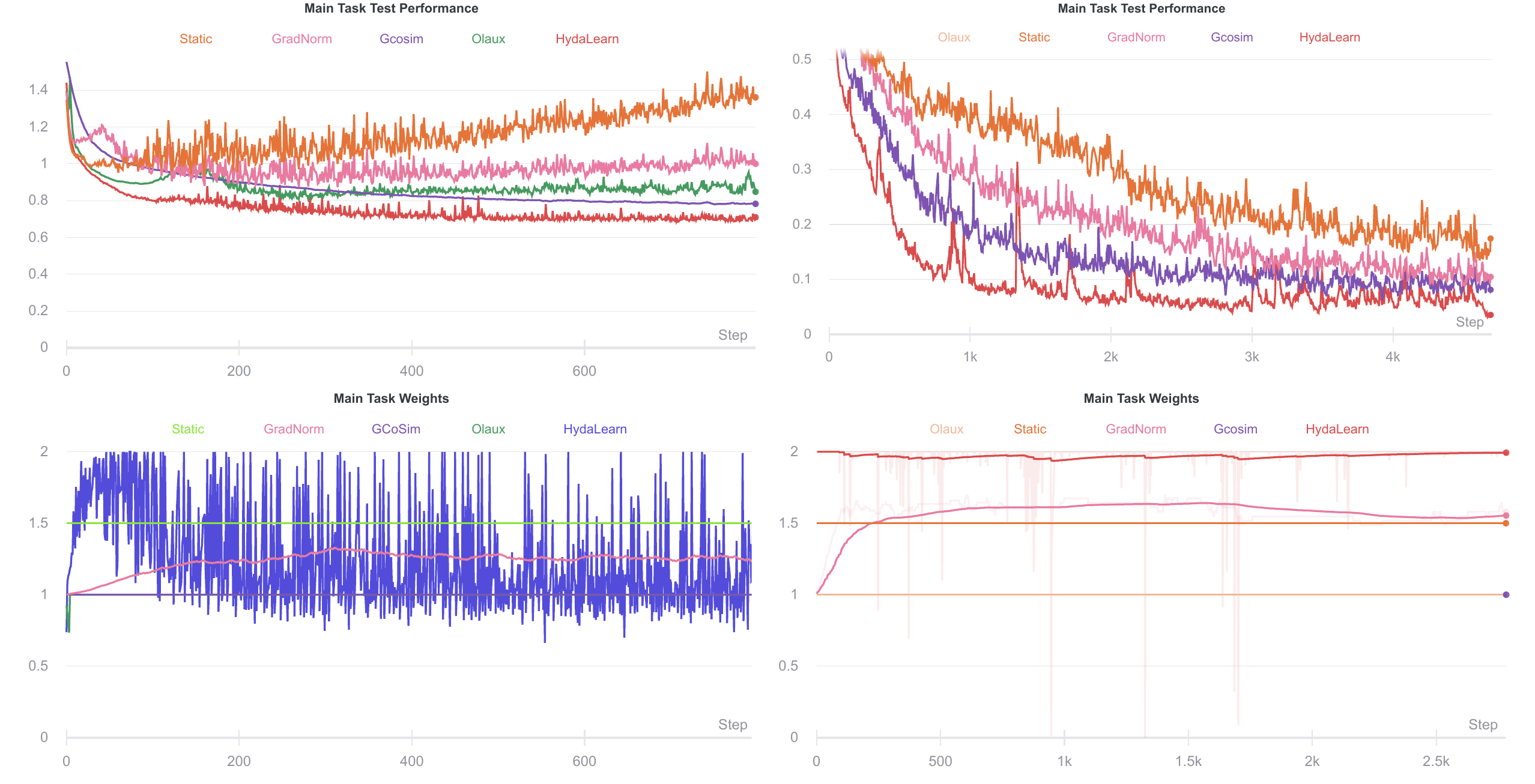}
		\caption{Results of the toy examples. The left and right hand graphs display results with a helpful (Exp1) and an unhelpful (Exp2) auxiliary task.}
		\label{fig_toyreg}
	\end{center}
\end{figure*}

As an implementation detail for HydaLearn, we allow down-scaling of so-called total learning rate $W$ when both $\delta$'s are negative, as follows: 
\begin{align}
    W' = W\left(1 + \exp\left(\dfrac{-w_a}{w_m}\right)\right)^{-1},
    \label{eq_lr_down}
\end{align}
where the training-step index $t$ is omitted for brevity. Both negative $\delta$'s value means neither of the task gradients for the given mini-batch is supporting improvement on the metric $\mc M_m$. One possibility was just skip updating the weights for that batch. But by doing we will not exploit information during that training step.  

In another implementation detail, fitting an exponential function to the $\delta$ ratio helps to amplify the difference in gain values on the metric $\mc M_m$ from gradients of the  corresponding tasks. More concretely, we do this as follows:
\begin{align}
    \dfrac{w_m}{w_a} = \left(\dfrac{\delta_m}{\delta_a}\right)^\beta,
    \label{eq_delta_expon}
\end{align}
where $\beta \in \mathbb{R}_{+}$ can be assumed as a hyper-parameter of the HydaLearn algorithm. The $\beta$ value can be used to control dynamics or stochasticity of weight changes from batch-to-batch.  

\textbf{Results}.
Results are shown in Fig. \ref{fig_toyreg}. From the figure we can see that for Exp1, the task weights are highly variable for HydaLearn. These drastic weight adjustments positively influences training, as HydaLearn outperforms the baseline methods. On the other hand, when $\mc T_a$ is harmful (Exp2), its weights are consistently low. As such, interference from harmful gradients coming from $\mc L_a$ is suppressed. 

The possible benefits of highly adaptive weighting methods are further evidenced through comparison with the baseline methods. Fig. \ref{fig_toyreg} shows that both the static baseline and GradNorm weightings consistently lie in between the trend weight values of HydaLearn. Yet, the MSE of HydaLearn is about 50\% higher. Olaux performs relatively poorly for both tasks. Its ever increasing weight allocation to the harmful auxiliary task is most surprising. Since Gcosim only considers the auxiliary task gradient when its cosine similarity with the main task gradient is positive, it is intuitively well suited for (Exp2). However, it does not completely rule out interference, as indicated by the superior test time performance of HydaLearn. The implementation details can be found in  Appendix \ref{annex_toyreg}.

\section{Experiments on Real World Datasets}
\subsection{Datasets and Tasks}
We apply HydaLearn to two pairs of supervised learning tasks from the following datasets:

\textbf{MIMIC-III.} 
The MIMIC-III database \cite{johnson2016mimic, saeed2011multiparameter} is comprised of de-identified data of over 60000 intensive care unit (ICU) stays. 
MIMIC-III is a popular resource for machine learning research for a variety of tasks, such as mortality prediction \cite{purushotham2017benchmark, mayaud2013dynamic, johnson2017reproducibility}, length of stay prediction \cite{purushotham2017benchmark, gentimis2017predicting}, and sepsis prediction \cite{nemati2018interpretable}. Since such tasks are often related, the database is also commonly used as a benchmark for multi-task learning algorithms \cite{suresh2018learning, harutyunyan2019multitask}.
Clinical data is often very noisy. Similarly, in the MIMIC dataset, the base features get recorded only sparsely, and at irregular intervals, requiring heavy imputation. 

We predict in-hospital mortality (classification) as our main task using features collected in the first 48h of stay. For the auxiliary task, the length of stay (regression) is used, which ends with either death or discharge from the hospital. As such, this experiment features a combination of a classification (area under the curve (AUC)) and a regression (MSE) loss. Both losses thus operate on a different scale, which can cause imbalances during learning. Furthermore, the dataset is high dimensional relative to its sample size. This can also impede learning, but can be mitigated through the regularizing effect of MTL. Further dataset details can be found in Appendix \ref{annex_mimic}.

\textbf{Fannie Mae Loan Performance.}
Data on mortgage default typically has an extreme class imbalance; the large majority of mortgage holders never default. Multitask learning is one way to amplify the signal of the minority class \cite{caruana2000learning}. As an auxiliary task, we propose to use prepayment prediction. By jointly learning both tasks, we can incorporate signal from future prepayments in the default model - a trick known as 'using the future to predict the present' \cite{ruder2016overview, caruana1996using}. For preprocessing and implementation details, see Appendix \ref{annex_fm}.

\textbf{Setup.}
For mortality prediction and default prediction, we use 4-shared layers with 48-neuron each, and 2-layer with 24-neuron encoders, respectively. The decoders are composed of 2 task-specific layers for each task, with 12 neurons for mortality prediction, and 24 for default prediction. These backbones were found to be suitable for these problems through a random search on the 'static baseline. The networks are optimized with vanilla mini-batch SGD with batch-size 16. This choice is motivated by the fact that alternative optimizers such as Adam or RMSprop scale the learning rate. In our experience, HydaLearn performs well with Adam too. The training time is chosen through early stopping on the validation set.

Since many of the baseline methods, and HydaLearn work on the gradient level, the learning rate is an important parameter. We perform a grid search over a range of sensible values of the learning rate for each method. Furthermore, we optimize the algorithm-specific hyper-parameters. The grids that were used, as well as the final picks for testing are reported in Table \ref{tabel_hyper_param} given in Appendix \ref{annex_toyreg}.

\subsection{Results}

Table \ref{Table_mimic_alg_comp} compares performance of different algorithms on the MIMIC dataset. The reported results are evaluated on the hold-out or test data-set. From the table we can see that HydaLearn gives better performance on the main task-metric compared to the STL baseline. Furthermore, HydaLearn surpasses the best performing algorithm from the recent state of art in multi-task learning with dynamic weight adaption for the MIMIC dataset. The same trend extends to the Fannie Mae dataset, as reported in Table \ref{Table_fm_alg_comp}. These performance gains on the real datasets, compared to the baslines, validate the effectiveness of the HydaLearn algorithm in dynamically extracting good values of the task weights.

Furthermore, from Fig. \ref{fig_wtone} we can derive that Olaux, HydaLearn and the static weight method take very different approaches to learning default. On average, Olaux learns a very balanced weighting, while the optimal Static combination weight has a very high weight for the main task, i.e. default prediction. In contrast, HydaLearn, gives a higher weight to the prepayment task on average. For batches with no default observations, it predominantly learns its shared representation through prepayment. When a batch is sampled for which the gradient of the default task is highly informative, i.e. one containing examples of the minority class, HydaLearn can adapt and immediately allocate high weight. Such dynamics would not be possible without the highly adaptive weights that characterize our algorithm.

\begin{table}[t!]
\caption{Experimental results for in-hospital-mortality prediction.}
\begin{center}
\begin{tabular}{| p{2cm} || p{2cm} | p{2cm} |}
\hline
 Model & AUC Metric & Std Deviation  \\ 
 \hline
 \hline
 HydaLearn & \textbf{0.839} & 0.003  \\
 \hline
 Gcosim & 0.774 & 0.018 \\
 \hline
 Olaux & 0.833 & 0.005 \\
 \hline
 GradNorm & 0.767 & 0.014 \\
 \hline
 Static & 0.834 & 0.007 \\
 \hline
 STL & 0.819 & 0.004 \\
 \hline

\end{tabular}
\end{center}
 \label{Table_mimic_alg_comp}

\end{table}

\begin{table}[t!]

\caption{Experimental results for default prediction.}
\begin{center}
\begin{tabular}{| p{2cm} || p{2cm} | p{2cm} |}
\hline
 Model & AUC Metric & Std Deviation  \\ 
 \hline
 \hline
 HydaLearn & \textbf{0.760} & 0.009  \\
 \hline
 Gcosim & 0.738 & 0.011 \\
 \hline
 Olaux & 0.743 & 0.015 \\
 \hline
 GradNorm & 0.734 & 0.021 \\
 \hline
 Static & 0.745 & 0.019 \\
 \hline
 STL & 0.732 & 0.010 \\
 \hline
\end{tabular}
\end{center}
\label{Table_fm_alg_comp}
\end{table}

\begin{figure}[t] 
	\begin{center}
		\includegraphics[scale=0.55]{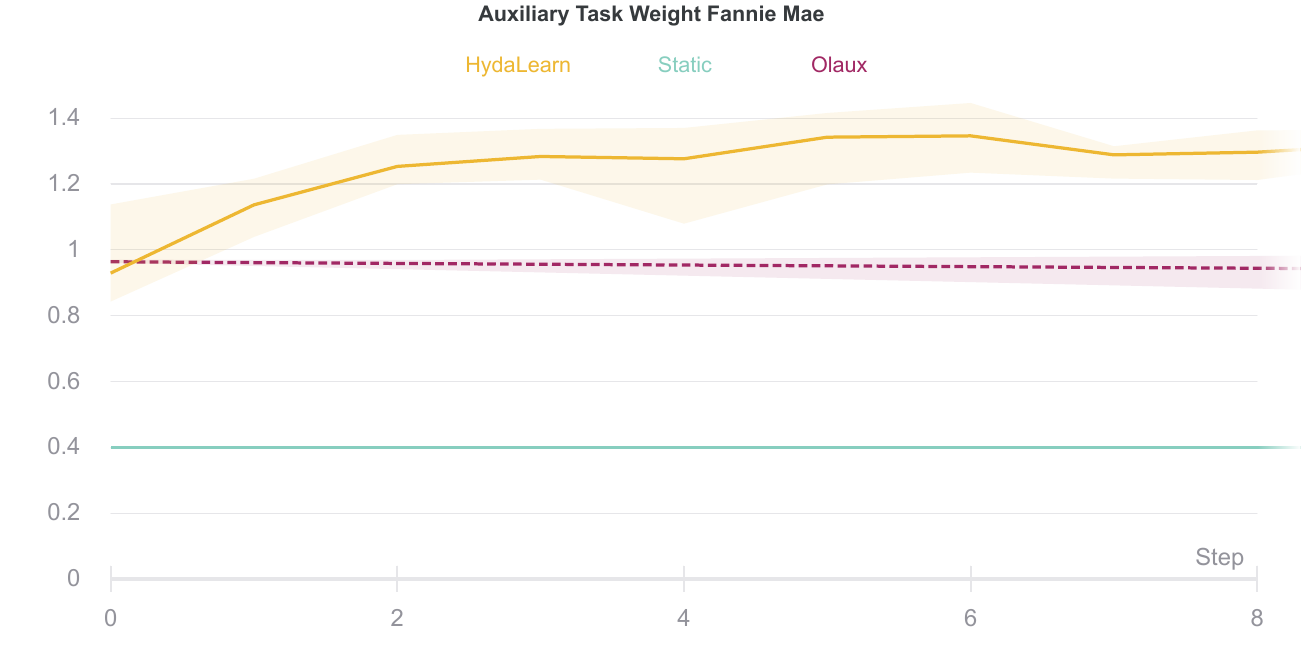}
		\caption{Average weights at the epoch level for the auxiliary task, Fannie Mae data}
		\label{fig_wtone}
	\end{center}
\end{figure}

\subsection{Parameter Impact Analysis} \label{Sec_impact_analysis}
As introduced in Section \ref{sec_alg_prop}, our model relies on so-called fake updates to gauge the usefulness of the task gradients for gain on the main task metric $\mc M_m$. Now in this section, we perform an impact study and study how model performance changes when we disable certain component of our algorithm. Concretely, we consider following four experiments, the results for which are presented in Table \ref{Table_impact_fm} and Table \ref{Table_impact_mimic};
\begin{itemize}
    \item{ExpImp-0:} Same configuration as in the preceding section. Key components of this configuration include: (c-i) normalizing gradients with Euclidean norm  during fake updates, (c-ii) computation of metric $\mc M_m$ on the validation dataset, and (c-iii) down-scaling of total learning rate as specified in \eqref{eq_lr_down}.  
    \item{ExpImp-1:} Compared to  ExpImp-0, we disable normalizing gradients by removing the normalization step (c-i) completely. rest of the configuraiton is same as in ExpImp-0. This implies that, in absence of normalization, if the gradients from the two fake update point in an equally beneficial direction for the main task, the gradient with a larger magnitude will dominate the final update. From  results in Table \ref{Table_impact_fm} and Table \ref{Table_impact_mimic}, we observe a small degradation, which is a little bit more on the MIMIC dataset. As such the degradation for both datasets are within the measurement uncertainty limit and thus negligible. However, the relatively larger change for the MIMIC dataset could partly be because the two tasks are of different type and essentially have completely different dynamic range.   
    \item{ExpImp-2:} All configuration parameters same as in ExpImp-0, except that the  $\delta$'s are computed over the training dataset. When calculating the $\delta$'s based on the validation set, signal from the validation set is incorporated in training. This has the undesirable side effect of increasing bias in validation metrics, complicating model selection and early stopping. If there is no performance trade-off, it makes sense to calculate $\delta$'s using the training set. However, this could encourage over-fitting on the training set. Based on the reported results, we see no evidence of over-fitting. But this because we used the entire training data in the computation of $\delta$'s. Although note reported in the tables, we observed strong evidences of over-fitting when relatively smaller subset of training dataset was used, which is inline with expectation.    
    \item{ExpImp-3:} In this experiment, we disabled the down-scaling of the total learning rate. All other configuration parameters were kept same as in ExpImp-0. From the results, we can see that the impact of this down-scaling feature is really negligible for the two datasets. However, in some of the toy examples on synthetic data we observed non-negligible gains. Further investigation is needed to ascertain the usefulness of this particular feature for other real world datasets.      
\end{itemize}

\begin{table}[t!]
\caption{Parameter impact study results for Fannie Mae data-set.}
\label{Table_impact_fm}
\begin{center}
\begin{tabular}{| p{2cm} || p{2cm} | p{2cm} |}
\hline
 Experiment & AUC Metric & Std Deviation  \\ 
 \hline
 \hline
 ExpImp-0& \textbf{0.760} & 0.009   \\
 \hline
 ExpImp-1 & 0.758 & 0.006 \\
 \hline
 ExpImp-2 & 0.750 & 0.01 \\
 \hline
 ExpImp-3 & 0.758 & 0.009 \\
 \hline
\end{tabular}
\end{center}
\end{table}

\begin{table}[t!]
\caption{Parameter impact study results for MIMIC data-set.}
\label{Table_impact_mimic}
\begin{center}
\begin{tabular}{| p{2cm} || p{2cm} | p{2cm} |}
\hline
 Experiment & AUC Metric & Std Deviation  \\ 
 \hline
 \hline
 ImpExp-0 & \textbf{0.839} & 0.003  \\
 \hline
 ImpExp-1 & 0.819 & 0.049 \\
 \hline
 ImpExp-2 & 0.839 & 0.004 \\
 \hline
 ImpExp-3 & 0.840 & 0.003 \\
 \hline
\end{tabular}
\end{center}
\end{table}

\section{Concluding Remarks}
In this work we presented a novel approach to dynamic task weighting in multi-task networks. We have shown that informing gradient weighting through an external metric can match and outperform the current state-of-the-art for two supervised learning tasks. The broader impact of these findings is twofold. First, implementing powerful multi-task models is made easier for practitioners, relieving them from the burden to find suitable static task loss weights, which may still perform worse than dynamic weighting. Second, current work offers a different approach to the dynamic weighting problem compared to previous work, motivated by mathematical grounding. Potential for future work is ample. For example, derivation of an approximation of the $\delta$ from the gradients, without computing extra backward passes would significantly reduce computational overhead. Furthermore, this framework can also be applied for high-dimensional problems such as object detection and NLP.

\bibliography{bibliography}

\vspace{12pt}

\begin{thebibliography}{49}
    \providecommand{\natexlab}[1]{#1}
    \providecommand{\url}[1]{\texttt{#1}}
    \providecommand{\urlprefix}{URL }
    \expandafter\ifx\csname urlstyle\endcsname\relax
      \providecommand{\doi}[1]{doi:\discretionary{}{}{}#1}\else
      \providecommand{\doi}{doi:\discretionary{}{}{}\begingroup
      \urlstyle{rm}\Url}\fi
    
    \bibitem[{Baesens et~al.(2005)Baesens, Van~Gestel, Stepanova, Van~den Poel, and
      Vanthienen}]{baesens2005neural}
    Baesens, B.; Van~Gestel, T.; Stepanova, M.; Van~den Poel, D.; and Vanthienen,
      J. 2005.
    \newblock Neural network survival analysis for personal loan data.
    \newblock \emph{Journal of the Operational Research Society} 56(9): 1089--1098.
    
    \bibitem[{Baesens, Van~Vlasselaer, and Verbeke(2015)}]{baesens2015fraud}
    Baesens, B.; Van~Vlasselaer, V.; and Verbeke, W. 2015.
    \newblock \emph{Fraud analytics using descriptive, predictive, and social
      network techniques: a guide to data science for fraud detection}.
    \newblock John Wiley \& Sons.
    
    \bibitem[{Bingel and S{\o}gaard(2017)}]{bingel2017identifying}
    Bingel, J.; and S{\o}gaard, A. 2017.
    \newblock Identifying beneficial task relations for multi-task learning in deep
      neural networks.
    \newblock \emph{arXiv preprint arXiv:1702.08303} .
    
    \bibitem[{Bottou(2010)}]{bottou2010large}
    Bottou, L. 2010.
    \newblock Large-scale machine learning with stochastic gradient descent.
    \newblock In \emph{Proceedings of COMPSTAT'2010}, 177--186. Springer.
    
    \bibitem[{Caruana(1993)}]{Caruana1993MultitaskLA}
    Caruana, R. 1993.
    \newblock Multitask Learning: A Knowledge-Based Source of Inductive Bias.
    \newblock In \emph{ICML}.
    
    \bibitem[{Caruana(1997)}]{caruana1997multitask}
    Caruana, R. 1997.
    \newblock Multitask learning.
    \newblock \emph{Machine learning} 28(1): 41--75.
    
    \bibitem[{Caruana(2000)}]{caruana2000learning}
    Caruana, R. 2000.
    \newblock Learning from imbalanced data: Rank metrics and extra tasks.
    \newblock In \emph{Proc. Am. Assoc. for Artificial Intelligence (AAAI) Conf},
      51--57.
    
    \bibitem[{Caruana, Baluja, and Mitchell(1996)}]{caruana1996using}
    Caruana, R.; Baluja, S.; and Mitchell, T. 1996.
    \newblock Using the future to" sort out" the present: Rankprop and multitask
      learning for medical risk evaluation.
    \newblock In \emph{Advances in neural information processing systems},
      959--965.
    
    \bibitem[{Chen et~al.(2017)Chen, Badrinarayanan, Lee, and
      Rabinovich}]{chen2017gradnorm}
    Chen, Z.; Badrinarayanan, V.; Lee, C.-Y.; and Rabinovich, A. 2017.
    \newblock Gradnorm: Gradient normalization for adaptive loss balancing in deep
      multitask networks.
    \newblock \emph{arXiv preprint arXiv:1711.02257} .
    
    \bibitem[{Collobert and Weston(2008)}]{collobert2008unified}
    Collobert, R.; and Weston, J. 2008.
    \newblock A unified architecture for natural language processing: Deep neural
      networks with multitask learning.
    \newblock In \emph{Proceedings of the 25th international conference on Machine
      learning}, 160--167.
    
    \bibitem[{Du et~al.(2018)Du, Czarnecki, Jayakumar, Pascanu, and
      Lakshminarayanan}]{du2018adapting}
    Du, Y.; Czarnecki, W.~M.; Jayakumar, S.~M.; Pascanu, R.; and Lakshminarayanan,
      B. 2018.
    \newblock Adapting auxiliary losses using gradient similarity.
    \newblock \emph{arXiv preprint arXiv:1812.02224} .
    
    \bibitem[{Gentimis et~al.(2017)Gentimis, Ala'J, Durante, Cook, and
      Steele}]{gentimis2017predicting}
    Gentimis, T.; Ala'J, A.; Durante, A.; Cook, K.; and Steele, R. 2017.
    \newblock Predicting hospital length of stay using neural networks on mimic iii
      data.
    \newblock In \emph{2017 IEEE 15th Intl Conf on Dependable, Autonomic and Secure
      Computing, 15th Intl Conf on Pervasive Intelligence and Computing, 3rd Intl
      Conf on Big Data Intelligence and Computing and Cyber Science and Technology
      Congress (DASC/PiCom/DataCom/CyberSciTech)}, 1194--1201. IEEE.
    
    \bibitem[{Golmant et~al.(2018)Golmant, Vemuri, Yao, Feinberg, Gholami,
      Rothauge, Mahoney, and Gonzalez}]{golmant2018computational}
    Golmant, N.; Vemuri, N.; Yao, Z.; Feinberg, V.; Gholami, A.; Rothauge, K.;
      Mahoney, M.~W.; and Gonzalez, J. 2018.
    \newblock On the computational inefficiency of large batch sizes for stochastic
      gradient descent.
    \newblock \emph{arXiv preprint arXiv:1811.12941} .
    
    \bibitem[{Goyal et~al.(2017)Goyal, Doll{\'a}r, Girshick, Noordhuis, Wesolowski,
      Kyrola, Tulloch, Jia, and He}]{goyal2017accurate}
    Goyal, P.; Doll{\'a}r, P.; Girshick, R.; Noordhuis, P.; Wesolowski, L.; Kyrola,
      A.; Tulloch, A.; Jia, Y.; and He, K. 2017.
    \newblock Accurate, large minibatch sgd: Training imagenet in 1 hour.
    \newblock \emph{arXiv preprint arXiv:1706.02677} .
    
    \bibitem[{Guo, Pasunuru, and Bansal(2019)}]{guo2019autosem}
    Guo, H.; Pasunuru, R.; and Bansal, M. 2019.
    \newblock Autosem: Automatic task selection and mixing in multi-task learning.
    \newblock \emph{arXiv preprint arXiv:1904.04153} .
    
    \bibitem[{Guo et~al.(2018)Guo, Haque, Huang, Yeung, and
      Fei-Fei}]{guo2018dynamic}
    Guo, M.; Haque, A.; Huang, D.-A.; Yeung, S.; and Fei-Fei, L. 2018.
    \newblock Dynamic task prioritization for multitask learning.
    \newblock In \emph{Proceedings of the European Conference on Computer Vision
      (ECCV)}, 270--287.
    
    \bibitem[{Harutyunyan et~al.(2017)Harutyunyan, Khachatrian, Kale, Steeg, and
      Galstyan}]{harutyunyan2017multitask}
    Harutyunyan, H.; Khachatrian, H.; Kale, D.~C.; Steeg, G.~V.; and Galstyan, A.
      2017.
    \newblock Multitask learning and benchmarking with clinical time series data.
    \newblock \emph{arXiv preprint arXiv:1703.07771} .
    
    \bibitem[{Harutyunyan et~al.(2019)Harutyunyan, Khachatrian, Kale, Ver~Steeg,
      and Galstyan}]{harutyunyan2019multitask}
    Harutyunyan, H.; Khachatrian, H.; Kale, D.~C.; Ver~Steeg, G.; and Galstyan, A.
      2019.
    \newblock Multitask learning and benchmarking with clinical time series data.
    \newblock \emph{Scientific data} 6(1): 1--18.
    
    \bibitem[{Hochreiter and Schmidhuber(1997)}]{hochreiter1997flat}
    Hochreiter, S.; and Schmidhuber, J. 1997.
    \newblock Flat minima.
    \newblock \emph{Neural Computation} 9(1): 1--42.
    
    \bibitem[{Hoffer, Hubara, and Soudry(2017)}]{hoffer2017train}
    Hoffer, E.; Hubara, I.; and Soudry, D. 2017.
    \newblock Train longer, generalize better: closing the generalization gap in
      large batch training of neural networks.
    \newblock In \emph{Advances in Neural Information Processing Systems},
      1731--1741.
    
    \bibitem[{Jean, Firat, and Johnson(2019)}]{jean2019adaptive}
    Jean, S.; Firat, O.; and Johnson, M. 2019.
    \newblock Adaptive Scheduling for Multi-Task Learning.
    \newblock \emph{arXiv preprint arXiv:1909.06434} .
    
    \bibitem[{Johnson, Pollard, and Mark(2017)}]{johnson2017reproducibility}
    Johnson, A.~E.; Pollard, T.~J.; and Mark, R.~G. 2017.
    \newblock Reproducibility in critical care: a mortality prediction case study.
    \newblock In \emph{Machine Learning for Healthcare Conference}, 361--376.
    
    \bibitem[{Johnson et~al.(2016)Johnson, Pollard, Shen, Li-wei, Feng, Ghassemi,
      Moody, Szolovits, Celi, and Mark}]{johnson2016mimic}
    Johnson, A.~E.; Pollard, T.~J.; Shen, L.; Li-wei, H.~L.; Feng, M.; Ghassemi,
      M.; Moody, B.; Szolovits, P.; Celi, L.~A.; and Mark, R.~G. 2016.
    \newblock MIMIC-III, a freely accessible critical care database.
    \newblock \emph{Scientific data} 3: 160035.
    
    \bibitem[{Kendall, Gal, and Cipolla(2018)}]{kendall2018multi}
    Kendall, A.; Gal, Y.; and Cipolla, R. 2018.
    \newblock Multi-task learning using uncertainty to weigh losses for scene
      geometry and semantics.
    \newblock In \emph{Proceedings of the IEEE conference on computer vision and
      pattern recognition}, 7482--7491.
    
    \bibitem[{Keskar et~al.(2016)Keskar, Mudigere, Nocedal, Smelyanskiy, and
      Tang}]{keskar2016large}
    Keskar, N.~S.; Mudigere, D.; Nocedal, J.; Smelyanskiy, M.; and Tang, P. T.~P.
      2016.
    \newblock On large-batch training for deep learning: Generalization gap and
      sharp minima.
    \newblock \emph{arXiv preprint arXiv:1609.04836} .
    
    \bibitem[{Kingma and Ba(2014)}]{kingma2014adam}
    Kingma, D.~P.; and Ba, J. 2014.
    \newblock Adam: A method for stochastic optimization.
    \newblock \emph{arXiv preprint arXiv:1412.6980} .
    
    \bibitem[{Li, Liu, and Wang(2019)}]{li2019gradient}
    Li, B.; Liu, Y.; and Wang, X. 2019.
    \newblock Gradient harmonized single-stage detector.
    \newblock In \emph{Proceedings of the AAAI Conference on Artificial
      Intelligence}, volume~33, 8577--8584.
    
    \bibitem[{Lin et~al.(2017)Lin, Goyal, Girshick, He, and
      Doll{\'a}r}]{lin2017focal}
    Lin, T.-Y.; Goyal, P.; Girshick, R.; He, K.; and Doll{\'a}r, P. 2017.
    \newblock Focal loss for dense object detection.
    \newblock In \emph{Proceedings of the IEEE international conference on computer
      vision}, 2980--2988.
    
    \bibitem[{Lin et~al.(2019)Lin, Baweja, Kantor, and Held}]{lin2019adaptive}
    Lin, X.; Baweja, H.; Kantor, G.; and Held, D. 2019.
    \newblock Adaptive Auxiliary Task Weighting for Reinforcement Learning.
    \newblock In \emph{Advances in Neural Information Processing Systems},
      4773--4784.
    
    \bibitem[{Maurer, Pontil, and Romera-Paredes(2016)}]{maurer2016benefit}
    Maurer, A.; Pontil, M.; and Romera-Paredes, B. 2016.
    \newblock The benefit of multitask representation learning.
    \newblock \emph{The Journal of Machine Learning Research} 17(1): 2853--2884.
    
    \bibitem[{Mayaud et~al.(2013)Mayaud, Lai, Clifford, Tarassenko, Celi, and
      Annane}]{mayaud2013dynamic}
    Mayaud, L.; Lai, P.~S.; Clifford, G.~D.; Tarassenko, L.; Celi, L. A.~G.; and
      Annane, D. 2013.
    \newblock Dynamic data during hypotensive episode improves mortality
      predictions among patients with sepsis and hypotension.
    \newblock \emph{Critical care medicine} 41(4): 954.
    
    \bibitem[{McCandlish et~al.(2018)McCandlish, Kaplan, Amodei, and
      Team}]{mccandlish2018empirical}
    McCandlish, S.; Kaplan, J.; Amodei, D.; and Team, O.~D. 2018.
    \newblock An empirical model of large-batch training.
    \newblock \emph{arXiv preprint arXiv:1812.06162} .
    
    \bibitem[{Nemati et~al.(2018)Nemati, Holder, Razmi, Stanley, Clifford, and
      Buchman}]{nemati2018interpretable}
    Nemati, S.; Holder, A.; Razmi, F.; Stanley, M.~D.; Clifford, G.~D.; and
      Buchman, T.~G. 2018.
    \newblock An interpretable machine learning model for accurate prediction of
      sepsis in the ICU.
    \newblock \emph{Critical care medicine} 46(4): 547.
    
    \bibitem[{Oksuz et~al.(2020)Oksuz, Cam, Kalkan, and Akbas}]{oksuz2020imbalance}
    Oksuz, K.; Cam, B.~C.; Kalkan, S.; and Akbas, E. 2020.
    \newblock Imbalance problems in object detection: A review.
    \newblock \emph{IEEE Transactions on Pattern Analysis and Machine Intelligence}
      .
    
    \bibitem[{Pang et~al.(2019)Pang, Chen, Shi, Feng, Ouyang, and
      Lin}]{pang2019libra}
    Pang, J.; Chen, K.; Shi, J.; Feng, H.; Ouyang, W.; and Lin, D. 2019.
    \newblock Libra r-cnn: Towards balanced learning for object detection.
    \newblock In \emph{Proceedings of the IEEE Conference on Computer Vision and
      Pattern Recognition}, 821--830.
    
    \bibitem[{Purushotham et~al.(2017)Purushotham, Meng, Che, and
      Liu}]{purushotham2017benchmark}
    Purushotham, S.; Meng, C.; Che, Z.; and Liu, Y. 2017.
    \newblock Benchmark of deep learning models on large healthcare mimic datasets.
    \newblock \emph{arXiv preprint arXiv:1710.08531} .
    
    \bibitem[{Rai and Daum{\'e}~III(2010)}]{rai2010infinite}
    Rai, P.; and Daum{\'e}~III, H. 2010.
    \newblock Infinite predictor subspace models for multitask learning.
    \newblock In \emph{Proceedings of the Thirteenth International Conference on
      Artificial Intelligence and Statistics}, 613--620.
    
    \bibitem[{Romera-Paredes et~al.(2012)Romera-Paredes, Argyriou, Berthouze, and
      Pontil}]{romera2012exploiting}
    Romera-Paredes, B.; Argyriou, A.; Berthouze, N.; and Pontil, M. 2012.
    \newblock Exploiting unrelated tasks in multi-task learning.
    \newblock In \emph{International conference on artificial intelligence and
      statistics}, 951--959.
    
    \bibitem[{Ruder(2016)}]{ruder2016overview}
    Ruder, S. 2016.
    \newblock An overview of gradient descent optimization algorithms.
    \newblock \emph{arXiv preprint arXiv:1609.04747} .
    
    \bibitem[{Ruder et~al.(2019)Ruder, Bingel, Augenstein, and
      S{\o}gaard}]{ruder2019latent}
    Ruder, S.; Bingel, J.; Augenstein, I.; and S{\o}gaard, A. 2019.
    \newblock Latent multi-task architecture learning.
    \newblock In \emph{Proceedings of the AAAI Conference on Artificial
      Intelligence}, volume~33, 4822--4829.
    
    \bibitem[{Saeed et~al.(2011)Saeed, Villarroel, Reisner, Clifford, Lehman,
      Moody, Heldt, Kyaw, Moody, and Mark}]{saeed2011multiparameter}
    Saeed, M.; Villarroel, M.; Reisner, A.~T.; Clifford, G.; Lehman, L.-W.; Moody,
      G.; Heldt, T.; Kyaw, T.~H.; Moody, B.; and Mark, R.~G. 2011.
    \newblock Multiparameter Intelligent Monitoring in Intensive Care II
      (MIMIC-II): a public-access intensive care unit database.
    \newblock \emph{Critical care medicine} 39(5): 952.
    
    \bibitem[{Sener and Koltun(2018)}]{sener2018multi}
    Sener, O.; and Koltun, V. 2018.
    \newblock Multi-task learning as multi-objective optimization.
    \newblock In \emph{Advances in Neural Information Processing Systems},
      527--538.
    
    \bibitem[{Smith and Le(2017)}]{smith2017bayesian}
    Smith, S.~L.; and Le, Q.~V. 2017.
    \newblock A bayesian perspective on generalization and stochastic gradient
      descent.
    \newblock \emph{arXiv preprint arXiv:1710.06451} .
    
    \bibitem[{Suresh, Gong, and Guttag(2018)}]{suresh2018learning}
    Suresh, H.; Gong, J.~J.; and Guttag, J.~V. 2018.
    \newblock Learning tasks for multitask learning: Heterogenous patient
      populations in the icu.
    \newblock In \emph{Proceedings of the 24th ACM SIGKDD International Conference
      on Knowledge Discovery \& Data Mining}, 802--810.
    
    \bibitem[{Sutton(1992)}]{sutton1992adapting}
    Sutton, R.~S. 1992.
    \newblock Adapting bias by gradient descent: An incremental version of
      delta-bar-delta.
    \newblock In \emph{AAAI}, 171--176.
    
    \bibitem[{Yang et~al.(2018)Yang, Zhang, Yu, Cai, and Luo}]{yang2018end}
    Yang, Z.; Zhang, Y.; Yu, J.; Cai, J.; and Luo, J. 2018.
    \newblock End-to-end multi-modal multi-task vehicle control for self-driving
      cars with visual perceptions.
    \newblock In \emph{2018 24th International Conference on Pattern Recognition
      (ICPR)}, 2289--2294. IEEE.
    
    \bibitem[{Yu et~al.(2020)Yu, Kumar, Gupta, Levine, Hausman, and
      Finn}]{yu2020gradient}
    Yu, T.; Kumar, S.; Gupta, A.; Levine, S.; Hausman, K.; and Finn, C. 2020.
    \newblock Gradient surgery for multi-task learning.
    \newblock \emph{arXiv preprint arXiv:2001.06782} .
    
    \bibitem[{Zhang et~al.(2018)Zhang, Liao, Rakhlin, Miranda, Golowich, and
      Poggio}]{zhang2018theory}
    Zhang, C.; Liao, Q.; Rakhlin, A.; Miranda, B.; Golowich, N.; and Poggio, T.
      2018.
    \newblock Theory of deep learning IIb: Optimization properties of SGD.
    \newblock \emph{arXiv preprint arXiv:1801.02254} .
    
    \bibitem[{Zhang et~al.(2014)Zhang, Luo, Loy, and Tang}]{zhang2014facial}
    Zhang, Z.; Luo, P.; Loy, C.~C.; and Tang, X. 2014.
    \newblock Facial landmark detection by deep multi-task learning.
    \newblock In \emph{European conference on computer vision}, 94--108. Springer.
    
    \end{thebibliography}

\section{Appendix}

\subsection{Function for Task Weight Optimization}\label{Ch6:Appendices:TBD}
Assume $\mc M_m$ is a differentiable function. The gain on the main task metric after taking a gradient descent step on weighted main task-loss objective $w_{m, t+1}\mc L_m(\theta_t)$ can be written as follows:  
\begin{align}
 \delta_{m, m, t+1} =& \mc M_m(\theta_{t+1}) -  \mc M_m(\theta_{t})\nonumber\\
 =& \mc M_m(\theta_t  - \alpha w_{m, t+1}\nabla_{\theta_t} \mc L_m(\theta_t)) -  \mc M_m(\theta_{t})\nonumber\\
 \approx& \mc M_m(\theta_t) - \alpha w_{m, t+1}\nabla_{\theta_t}\mc M_m(\theta_t)^T \nabla_{\theta_t} \mc L_m(\theta_t)\nonumber\\
 &- \mc M_m(\theta_{t}) \nonumber\\
 &= -\alpha w_{m, t+1}\nabla_{\theta_t}\mc M_m(\theta_t)^T \nabla_{\theta_t} \mc L_m(\theta_t),
 \label{eq_delta_m}
\end{align}
where the approximation follows from the first-order Taylor series expansion. Similarly the gain on the main task metric after taking a gradient descent step on weighted auxiliary task-loss objective $w_{a, t+1}\mc L_a(\theta_t)$ can be written as follows:  
\begin{align}
 \delta_{m, a, t+1} =& \mc M_m(\theta_{t+1}) -  \mc M_m(\theta_{t})\nonumber\\
 =& \mc M_m(\theta_t  - \alpha w_{a, t+1}\nabla_{\theta_t} \mc L_a(\theta_t)) -  \mc M_m(\theta_{t})\nonumber\\
 \approx& \mc M_m(\theta_t) - \alpha w_{a, t+1}\nabla_{\theta_t}\mc M_m(\theta_t)^T \nabla_{\theta_t} \mc L_a(\theta_t)\nonumber\\
 &- \mc M_m(\theta_{t}) \nonumber\\
 =& -\alpha w_{a, t+1}\nabla_{\theta_t}\mc M_m(\theta_t)^T \nabla_{\theta_t} \mc L_a(\theta_t),
 \label{eq_delta_a}
\end{align}
where again the approximation is obtained by applying the first-order Taylor series expansion. Based on \eqref{eq_delta_m} and \eqref{eq_delta_a}, we can write
\begin{align}
 \dfrac{\delta_{m, m, t+1}}{\delta_{m, a, t+1}} = \left(\dfrac{w_{m, t+1}}{w_{a, t+1}}\right)\left(\dfrac{\nabla_{\theta_t}\mc M_m(\theta_t)^T \nabla_{\theta_t} \mc L_m(\theta_t)}{\nabla_{\theta_t}\mc M_m(\theta_t)^T \nabla_{\theta_t} \mc L_a(\theta_t)}\right)
\end{align}

\begin{align}
\dfrac{w_{m, t+1}}{w_{a, t+1}} = \left(\dfrac{\delta_{m, m, t+1}}{\delta_{m, a, t+1}}\right) \left(\dfrac{\nabla_{\theta_t}\mc M_m(\theta_t)^T \nabla_{\theta_t} \mc L_a(\theta_t)}{\nabla_{\theta_t}\mc M_m(\theta_t)^T \nabla_{\theta_t} \mc L_m(\theta_t)}\right)
\end{align}

Assuming  
\begin{align}
\nabla_{\theta_t}\mc M_m(\theta_t)^T \nabla_{\theta_t} \mc L_a(\theta_t) \approx \nabla_{\theta_t}\mc M_m(\theta_t)^T \nabla_{\theta_t} \mc L_m(\theta_t),\nonumber
\end{align}
we get
\begin{align}
 \dfrac{w_{m, t+1}}{w_{a, t+1}} \approx \dfrac{\delta_{m, m, t+1}}{\delta_{m, a, t+1}}. 
\end{align}
The underlying assumption is expected to hold when parameter vector dimension reasonably large, and loss functions are on the same scale.

\subsection{Toy Example Details}\label{annex_toyreg}
Following \cite{chen2017gradnorm}, we sample two regression tasks from the following functions:
\begin{equation}
    f_i(\textbf{x}) = \sigma_i \tanh{((\mathbf{B} + \mathbf{\epsilon_i})\mathbf{x})},
    \label{eq_toyreg}
\end{equation}
where $\mathbf{x}$ is the input vector. $\mathbf{\epsilon_i}$ and $\mathbf{B}$ are constant matrices representing a task-dependent and a shared component, respectively. The $\sigma_i$ linearly affect the scale of the tasks.

\begin{table}[h]
\caption{Toy example dataset details}
\begin{tabular}{|l|l|l|l|}
\hline
Toy Example & Training/Val/Test size & Inputs $\mathbf{x}$ & Outputs\\
\hline
\hline
$T_(helpful)$& 10000/2000/2000        & 75         & 25      \\
\hline
$T_(unhelpful)$ & 1000/200/200           & 25         & 5      \\
\hline
\end{tabular}
\end{table}

For both experiments $\mathbf{B}$ was sampled I.I.D. from a Gaussian with mean 0 and variance 10 , and $\mathbf{\epsilon}_i$ from a Gaussian with mean 0 and variance 3.5. To represent common scale differences in tasks. The scaling parameters $\sigma_m$ and $\sigma_a$ were set to 1 and 10, respectively. 

\begin{table}[h]
\centering
\caption{Final Hyperparameters}
\label{tabel_hyper_param}
\begin{tabular}{|l|l|l|l|l|l|l|}
\hline

        &    & HydaLearn & Gcosim & Olaux & Static & STL    \\
\hline
\hline
Toy 1/2 & AS & 6         & /      & 5     & 1.5    & /      \\
        & Lr & 0.01      & 0.01   & 0.01  & 0.01   & 0.01   \\
\hline
MIMIC   & AS & 3         & /      & 5     & 1.6    & /      \\
        & Lr & 0.005     & 0.001  & 0.01  & 0.005  & 0.0025 \\
\hline
FM      & AS & 3         & /      & 5     & 1.6    & /      \\
        & Lr & 0.05      & 0.05   & 0.01  & 0.01   & 0.001  \\
\hline
\end{tabular}
\end{table}

\begin{table}[h]
\centering
\caption{Base Feature Sets}
\begin{tabular}{|l|l|}
\hline
Fannie Mae Features           & MIMIC base features                             \\ \hline
       \hline
seller                & capillary refill rate              \\
servicer              & glascow coma scale eye opening     \\
loan purpose          & glascow coma scale motor response  \\
first time home-buyer & glascow coma scale total           \\
channel               & glascow coma scale verbal response \\
PPM                   & diastolic blood pressure           \\
occupancy             & fraction inspired oxygen           \\
product type          & glucose                            \\
property state        & hearth rate                        \\
property type         & height                             \\
MSA                   & mean blood pressure                \\
units                 & oxygen saturation                  \\
CLTV                  & respiratory rate                   \\
DTI                   & systolic blood pressure            \\
UPB                   & temperature                        \\
LTV                   & weight                             \\
interest rate         & pH                                 \\
loan term             &                                    \\
num borrowers         &                                    \\
servicer              &                                    \\
credit score          &                                    \\
6 months to maturity  &                                    \\ \hline
\end{tabular}
\end{table}

 For GradNorm and Olaux we take the recommended values \cite{chen2017gradnorm, lin2019adaptive}. Note that for the experiments with real data we perform hyperparameter optimisation.

\subsection{Preprocessing and Implementation: MIMIC}\label{annex_mimic}

Only episodes that last longer than 48 hours are considered. For in-hospital-mortality preprocessing, we follow the same approach as for the logistic regression baseline in \cite{harutyunyan2017multitask} \footnote{https://github.com/YerevaNN/mimic3-benchmarks} to enrich 17 base-feature dataset. First, a given sequence is divided into 7 sub-sequences. Next, features are extracted for each subsequence, based on statistical characteristics of the original timeseries variables. Specifically; mean, standard deviation, minimum, maximum, skewness and number of measurements \cite{harutyunyan2017multitask}. This procedure yields 714 features (7 subsequences X 6 statistic features X 17 base features).

For all the models, we used the same encoder/decoder set-up. We used a random sweep over a range of possible configurations of the 'static' baseline to determine the backbone and batchsize parameters. The shared layers and both task-specific heads of the network consist of 4 layers with 48 neurons, and 2 layers with 24 neurons, respectively. The learning rate, early stopping point, and algorithm-specific hyperparameters were decided using gridsearch.

\subsection{Preprocessing and implementation: Fannie Mae}\label{annex_fm}
We use a slice of the Fannie and Mae dataset\footnote{https://www.fanniemae.com/portal/funding-the-market/data/loan-performance-data.html}. It includes data on over one million mortages. Improved sample efficiency through MTL has diminishing returns. As such, it makes sense to subsample the dataset to a reasonable size. Consequently, we take a uniformly sampled, I.I.D. slice of 10000 data points, containing mortgages that were accepted between 2000 and 2009. For prediction, we use the status at the start of 2010 to predict occurrence of default and prepayment over the next twelve months. 

The continuous and categorical features were standardized and onehot encoded, respectively. The resulting 138 features are used in our experiments. The basic backbone architecture is the same for all models used in the experiments with the Fannie Mae dataset. This backbone consists of 2 24-neuron shared layers, and two 2-layer 12-neuron task-specific heads. Again, learning rate, early stopping, and algorithm-specific hyperparameters are determined for each baseline separately by grid-search.

\end{document}